\documentclass{llncs}

\usepackage{times}
\usepackage{setspace}
\usepackage{amsfonts}
\usepackage{amsmath}
\usepackage{amssymb}
\usepackage{mathptmx}
\usepackage{nicefrac}
\usepackage{url}
\usepackage{pifont}
\usepackage{fancyhdr}
\usepackage{multirow}
\usepackage{lineno}
\usepackage[pdftex]{graphicx}
\usepackage{algorithm}
\usepackage{algorithmic}
\usepackage{enumitem}
\usepackage{adjustbox}
\usepackage{hyperref}
\usepackage{color}
\usepackage{xspace}
\usepackage[numbers,sectionbib]{natbib}

\usepackage{tikz}
\usepackage{pgf}
\usetikzlibrary{arrows}
\usetikzlibrary{positioning}
\usetikzlibrary{calc}
\usetikzlibrary{shapes,decorations}
\hypersetup{ 
colorlinks,%
citecolor=black,%
filecolor=black,%
linkcolor=black,%
urlcolor=blue
}

\usepackage{boxedminipage}
\newcommand{\pbDef}[3]{%
\noindent
\begin{center}
\begin{boxedminipage}{0.99 \columnwidth}
#1
\smallskip\\
\begin{tabular}{l p{0.75 \columnwidth}}
Input: & #2\\
Question: & #3
\end{tabular}
\end{boxedminipage}
\end{center}
}

\newcommand{\reals}{\ensuremath{\mathbb{R}}}

\definecolor{blue}{RGB}{0, 93, 170}			
\definecolor{green}{RGB}{0, 102, 0}					
\definecolor{red}{RGB}{250,60,50}				

\newcommand{\OMIT}[1]{} 

\newtheorem{observation}[theorem]{Observation}

\newcommand{\ocmax}{\ensuremath{c^{O}_{max}}}
\newcommand{\ocmin}{\ensuremath{c^{O}_{min}}}
\newcommand{\acmax}{\ensuremath{c^{N}_{max}}}
\newcommand{\acmin}{\ensuremath{c^{N}_{min}}}

\newcommand{\wrs}{\ensuremath{\boldsymbol{\omega}}}
\newcommand{\owa}{\ensuremath{\boldsymbol{\alpha}}}
\newcommand{\all}{\ensuremath{\mathbb{A}}}
\newcommand{\uv}{\ensuremath{\textbf{u}}}
\newcommand{\xv}{\ensuremath{\textbf{x}}}
\newcommand{\maxv}{\ensuremath{\Delta}}

\newcommand{\sig}{\ensuremath{\boldsymbol{\sigma}}}
\newcommand{\sumowa}{$\Sigma$-OWA~}

\title{The Conference Paper Assignment Problem:\\ Using Order Weighted Averages to Assign Indivisible Goods}

\author{
Jing Wu Lian\inst{1} 
\and Nicholas Mattei\inst{2}
\and Renee Noble\inst{1} 
\and Toby Walsh\inst{3}
}

\institute{
Data61, CSIRO and UNSW, Sydney, Australia \\ \email{renee.noble@data61.csiro.au, lianjingwu@gmail.com}
\and
IBM T.J. Watson Research Center, New York, USA \\ \email{n.mattei@ibm.com}
\and
Data61, CSIRO, UNSW, and TU Berlin, Berlin, Germany \\ \email{toby.walsh@data61.csiro.au}}

\begin{document}

\maketitle

\begin{abstract}
Motivated by the common academic problem of allocating papers to referees for conference reviewing we propose a novel mechanism for solving the assignment problem when we have a two sided matching problem with preferences from one side (the agents/reviewers) over the other side (the objects/papers) and both sides have capacity constraints.  The assignment problem is a fundamental problem in both computer science and economics with application in many areas including task and resource allocation. We draw inspiration from multi-criteria decision making and voting and use order weighted averages (OWAs) to propose a novel and flexible class of algorithms for the assignment problem. We show an algorithm for finding an \sumowa assignment in polynomial time, in contrast to the NP-hardness of finding an egalitarian assignment.  Inspired by this setting we observe an interesting connection between our model and the classic proportional multi-winner election problem in social choice.
\end{abstract}

\section{Introduction}
Assigning indivisible items to multiple agents is a fundamental problem in many fields including computer science, economics and operations research. Algorithms for matching and assignment are used in a variety of application areas including allocating runways to airplanes, residents to hospitals, kidneys to patients \cite{DPS14a}, students to schools \cite{BuCa12a}, assets to individuals in a divorce, jobs to machines, and tasks to cloud computing nodes \cite{Man13a}.  Understanding the properties of the underlying algorithms is an important aspect to ensuring that all participating agents are happy with their allocations and do not attempt to misrepresent their preferences; a key area of study for computational social choice \cite{BCELP16a}.

An area that is near to many academics' hearts is the problem of allocating papers to referees for peer review. The results of grant, journal, and conference reviewing can have significant impact on the careers of scientists. Ensuring that papers and proposals are reviewed by the most qualified/interested referees most is part of ensuring that items are treated properly and all participants support the outcome of the processes.  Making sure these processes work for both the proposers and the reviewers is important and methods for improving peer review have been proposed and discussed in AI \cite{PrFl16} and broadly across the sciences \cite{MeSa09a}. 

There are a number of ways one can improve the quality of peer review \cite{PrFl16}.  First is to ensure that reviewers are not incentivized to misreport their reviews for personal gain.  Along this line there has been significant interest recently in strategyproof mechanisms for peer review \cite{ALMR+15a}.  Unfortunately, the method that we discuss in this paper is not strategyproof.  Another way is to ensure that reviewers are competent to provide judgements on the papers they are assigned.  The Toronto Paper Matching System \cite{ChZe13a} is designed to improve the process from this paper-centric model.  A third alternative, and the one we focus on in this study, is ensuring that reviewers are happy with the papers they are asked to review.  This is fundamentally a question about the optimization objectives of the assignment functions used.

Formally, we study the Conference Paper Assignment Problem (CPAP) \cite{GoSl07a} which is a special of the Multi-Agent Resource Allocation Problem (MARA) \cite{BoChLa16}, and propose a novel assignment, the \sumowa assignment.  In the CPAP setting we have a two-sided market where on one side the agents/reviewers have preferences over the other side, the objects/papers, and both sides have (possibly infinite) upper and lower capacities. A fundamental tension in assignment settings is the tradeoff between maximizing the social welfare, also know as the utilitarian maximal assignment and the Rawlsian \cite{Rawls71a} fairness concept of maximizing the utility of the worst off agent, known as the egalitarian maximal assignment. These two ideas are incompatible optimization objectives and diverge in a computational sense as well: computing the utilitarian assignment for additive utilities can be done in polynomial time, while computing the egalitarian assignment is NP-complete \cite{DeHi88a}.  This, perhaps, could be the reason that implementers of large conference paper assignment software often opt for utilitarian assignments, as is supposedly the case for EasyChair \cite{GKKM+10a}.\footnote{This is technically unsubstantiated as when the authors contacted EasyChair to understand the assignment process we were told, \emph{``We do not provide information on how paper assignment in EasyChair is implemented. The information in Garg et.al. may be incorrect or out of date - none of the authors worked for EasyChair, they also had no access to the EasyChair code.''}}  However, it is also not clear if an egalitarian assignment is desirable for CPAP.

\smallskip
\noindent
\textbf{Contributions.}
We establish a motivation for using OWA vectors in the assignment setting and define a novel notion of allocation, the \sumowa assignment.
We give algorithm to compute an \sumowa maximal assignment in polynomial time and we show that the \sumowa objective generalizes the utilitarian objective.
We show that \sumowa assignments satisfy a notion of Pareto optimality w.r.t. the pairwise comparisons of the objects by the agents.
We implement an algorithm for \sumowa assignments and perform experiments on real world conference paper assignment data.

\section{Preliminaries}

From here we will use the more general notation agents/objects to describe our setting.  In assignment settings each agent provides their preference over the objects as a reflexive, complete, and transitive preference relation (weak order) over the set of objects, $\succsim_i$.  We do not assume that $\succsim_i$ is complete; it is possible that some agents may have conflicts of interest or have no preference for a particular object; this assumption is often called ``having unacceptable objects'' in the literature \cite{Man13a}.

In many real-world CPAP settings there are a fixed number of equivalence classes into which agents are asked to place the objects \cite{MaWa13a}.  We assume that the number of equivalence classes (ranks) of objects are given as input to the problem and agents tell us within which rank each objects belong.  Agents also provide a decreasing utility value for each rank\footnote{We assume that agents can give any utilities as input.  However, often the utilities are restricted to be the same, i.e., Borda utilities in conference paper bidding, or come from some fixed budget, i.e., bidding fake currency as in course allocation at Harvard \cite{BuCa12a}.}. Our main result can be extended to the case where the number of equivalence classes is not fixed.

Formally, the CPAP problem is defined by $(N, O, \succsim, \uv, \maxv)$: a set of $n$ agents $N = \{a_1, \ldots, a_n\}$; a set of $m$ objects $O = \{o_1, \ldots, o_m\}$; for each $i \in N$, a reflexive and transitive preference relation (weak order) over the set of objects, $\succsim_i$, divided into $\maxv$ equivalence classes (ranks); and for each $i \in N$ a utility vector $\uv_i$ is of length $\maxv$ and assigns a decreasing utility $\uv_i(k) \rightarrow \reals$ for each $k \in [1, \maxv]$, i.e., $\uv_i(1) > \uv_i(2) > \ldots > \uv_i(\maxv)$.  Let $r_i(j)$ be the rank of object $j$ for $i$ and $\uv_i(r_i(j))$ denote the value of $i$ for $j$.

\subsection{Side Constraints and Feasible Assignments}

There are two practical constraints that we include in our model, making our model more general than the standard MARA or CPAP problems studied in computer science \cite{BoChLa16}: upper and lower capacities on both the agents and objects. 

\textbf{Agent Capacity:} each agent $i \in N$ has (possibly all equal) upper and lower bound on their capacity, the number of objects they can be allocated, $\acmin(i)$ and $\acmax(i)$.  

\textbf{Object Capacity:} each object $j \in O$ has a (possibly all equal) upper and lower bound on the number of agents assigned to it, $\ocmin(j)$ and  $\ocmax(j)$, respectively.  

We can now define a feasible assignment $A$ for an instance $(N, O, \succ, \uv, \maxv)$. For a given assignment $A$, let $A(i, :)$ denote the set of objects assigned to agent $i$ in $A$, let $A(:, j)$ denote the set of agents assigned to object $j$, and let $|\cdot|$ denote the size (number of elements) of a set or vector.  A feasible assignment $A$ must obey:
\begin{align*}
[\forall i \in N: \acmin(i) \leq |A(i, :)| \leq \acmax(i)]  \wedge
[\forall j \in O:  \ocmin(j) \leq |A(:, j)| \leq \ocmax(j)].
\end{align*}
We write the set of all feasible assignments for an instance as $\all(N, O, \succ, \uv, \maxv)$.\footnote{We will omit the arguments when they are clear from context.}

\subsection{Individual Agent Evaluation}

We first formalize how an individual agent evaluates their assigned objects. Each feasible assignment $A \in \all$ gives rise to a \emph{signature vector} for each agent $i \in N$; intuitively the signature vector is the number of objects at each rank assigned to $i$.  Formally let $\sig_{i}(A) = (\sig_{i,1}(A), \ldots, \sig_{i,\maxv}(A))$ where $\sig_{i,l}(A) = | \{j  \in A(i, :)  | r_i(j) = l \} |$ for each $l \in [1, \ldots, \maxv]$.

For indivisible (discrete) objects the lexicographic relation can be modeled by the additive utility relation by setting the agent utilities to high enough values. Formally, if the utility for rank $i < j$ is $\uv(i) > \uv(j) \cdot m$ then the lexicographic and additive utility relations are the same, i.e., no matter how many additional objects of rank $j$ the agent receives, one additional object of rank $i$ is more preferred.  We now define the relations that a referee might consider between assignments A and B.

\textbf{Lexicographic:} An agent $i$ lexicographically prefers $A$ to $B$ if $\sig_i(A)$ comes before $\sig_i(B)$ in the lexicographic order.  That is, there is an index $1 \geq l \geq \maxv$ such that for all $k > l$ we have $\sig_{i,k}(A) = \sig_{i,k}(B)$ and $\sig_{i,l} (A) > \sig_{i, l} (B)$; i.e., $i$ receives at least one more paper of a higher rank in $A$ than in $B$.  The lexicographic relation over vectors has a long history in the assignment literature \cite{Fish74a}.  

\textbf{Additive Utility:} An agent $i$ prefers assignment $A$ to $B$ if he has more additive utility for the objects assigned to him in $A$ than in $B$.  Formally, (and slightly abusing notation)
$
\uv_i(A) = \sum_{j \in A(i, :)} \uv_i (r_i(j))  > \sum_{j \in B(i, :)} \uv_i (r_i(j))$, or an alternative formulation using the dot product, $\uv_i(A) = \uv_i \cdot \sig_i(A) > \uv_i \cdot \sig_i(B).$

\subsection{Overall Assignment Evaluation}
In the literature there are several optimization objectives defined over an assignment that an implementer may wish to consider.  We limit our discussion to the two classical notions below.
Additional discussion of objectives, including the imposition of various fairness criteria for the CPAP setting can be found in \citet{GKKM+10a} and for the MARA setting see e.g., \citet{BoLe16a}.

\textbf{Utilitarian Social Welfare Maximal Assignment:} Often called the utilitarian assignment, we want to maximize the total social welfare over all the agents. An assignment is a utilitarian assignment if it satisfies: 
\begin{align*} 
& \arg\max_{A \in \all} \sum_{i \in N} \sum_{ j \in A(i,:)} \uv_i (r_i(j)) = \arg\max_{A \in \all} \sum_{ i \in N} \uv_i \cdot \sig_i(A).
\end{align*} 

\textbf{Egalitarian Social Welfare Maximal Assignment:} Often called the egalitarian assignment, we want to enforce the Rawlsian notion of fairness by making sure that the worst off referee is as happy as  possible, i.e., maximize the utility of the least well off agent.  Formally, 
$$\arg\max_{A \in \all} \min_{i \in N} \sum_{j \in A(i, :)} \uv_i (r_i(j)) =
\arg\max_{A \in \all}\min_{i \in N} \uv_i \cdot \sig_i(A).$$

In the discrete MARA and CPAP setting where objects are not divisible, the problem of finding an egalitarian assignment is NP-hard \cite{DeHi88a} while finding a utilitarian assignment can be done in polynomial time \cite{BoChLa16}.

\section{Background and Related Work}

One and two sided matching and assignment problems have been studied in economics \cite{RS92a} and computer science \cite{Man13a,BCELP16a} for over 40 years.  Matching and assignment have many applications including kidneys exchanges \cite{DPS14a} and school choice \cite{APR05a}. Our problem is often called the multi-agent resource allocation (MARA) problem in computer science \cite{BoChLa16} The papers to referees formulation of this problem has some additional side constraints common in the economics literature, but not as common in computer science \cite{Man13a}.  In the economics literature the \textit{Workers-Firms} problem is the most closely related analogue to our problem, modeling many-many matchings with capacities \cite{KlMaRo16}.

The conference paper assignment has been studied a number of times over the years in computer science \cite{GoSl07a}, as has defining and refining notions of fairness for the assignment vectors in multi-agent allocation problems \cite{GoPe10a}. We build off the work of \citet{GKKM+10a}, who extensively study the notion of fair paper assignments, including lexi-min and rank-maximal assignments, within the context of conference paper assignment.  \citet{GKKM+10a} show that for the setting we study, finding an egalitarian optimal assignment and finding a leximin optimal assignment are both NP-hard when there are three or more equivalence classes; and polynomial time computable when there are only two.  They also provide an approximation algorithm for leximin optimal assignments. 
We know that if the capacity constraints are hard values, i.e., each reviewer must review $\leq x$ papers and each paper must receive exactly $y$ reviews, then the resulting version of capacitated assignment is NP-hard \cite{LWPY13a}.  Answer set programming for CPAP was studied by \citet{ADLR16a}; they encode the CPAP problem in ASP and show that finding a solution that roughly correspond to the leximin optimal and egalitarian solutions can be done in reasonable time for large settings ($\approx 100$ agents).

CPAP also receives considerable attention in the recommender systems \cite{CKR09a} and machine learning \cite{CZB12a} communities. Often though, this work takes the approach of attempting to infer a more refined utility or preference model in order to distinguish papers. Fairness and efficiency concerns are secondary.  A prime example of this is the Toronto Paper Matching System designed by \citet{ChZe13a}. This system attempts to increase the accuracy of the matching algorithms by having the papers express preferences over the reviewers themselves; where these preferences are inferred from the contents of the papers. 

We make use of Order weighted averages (OWAs), often employed in multi-criteria decision making \cite{Yager88a}.  OWAs have recently received attention in computational social choice for voting and ranking \cite{GLMP14a}, finding a collective set of items for a group \cite{SFL16a}, and multi-winner voting with proportional representation  \cite{EFSS14a,ElIs15a}. The key difference between CPAP and voting using OWAs in the ComSoc literature is that CPAP does not select a set of winners that all agents will share.  Instead, all agents are allocated a possibly disjoint set of objects. 

%
%

\section{\sumowa Assignments}
We now formally define OWAs and their use for defining assignment objectives.  We will discuss alternative formulations of \sumowa which have been studied.

An order weighted average (OWA) is a function defined for an integer $K$ as a vector $\owa^{(K)} = (\alpha_1, \ldots, \alpha_K)$ of $K$ non-negative numbers. Let $\xv = (x_1, \ldots, x_K)$ be a vector of $K$ numbers and let $\xv^{\downarrow}$ be the non-increasing rearrangement of $\xv$, i.e. $\xv^{\downarrow} = \xv^{\downarrow}_{1} \geq \xv^{\downarrow}_{2} \geq \ldots \geq \xv^{\downarrow}_{K}$.  Then we say: $$OWA_{\owa}(\xv) = \owa \cdot \xv^{\downarrow} =  \sum^{K}_{i=1} \alpha_i \cdot \xv^{\downarrow}_i.$$

In order to apply OWAs to our setting we need to define the \emph{weighted rank signature} of an assignment.  Let $\wrs_i(A)$ be defined as the sorted vector of utility that a referee gets from an assignment $A$.  Formally, $$ \wrs_i(A) = sort(\{\forall j \in A(i, :): \uv_i(r(j))\}).$$ For example, if $A(i, :)$ included two objects with utility $3$, one of utility $1$, and one of utility $0$, we would have $$\wrs_i(A) = (3, 3, 1, 0).$$

Our inspiration for applying OWAs comes from a multi-winner voting rule known as Proportional Approval Voting (PAV) \cite{Kilg10a,AGGM+15a,ABC+14c}.  In approval voting settings, each agent can approve of as many candidates as they wish.  Under the standard approval voting (AV) method, all approvals from each agent assign one point to the candidate for which they are cast.  However, this can lead to a number of pathologies described by \citet{AGGM+15a} and it intuitively does not seem fair; once a candidate that you like has been selected to the winning set your next candidate selected to the winning set should seemingly count less.  Hence in PAV, which is designed to be more fair \cite{ABC+14c}, a voter's first approval counts for a full point, the second for $\nicefrac{1}{2}$, the next for $\nicefrac{1}{3}$, and on as a harmonically decreasing sequence.  

Transitioning this logic to the CPAP setting, we were motivated to find a way to distribute objects to agents that increases the number of agents who receive their top ranked objects.  This is the logic of PAV: once you get a candidate into the winning set, you should count less until everyone else has a candidate in the winning set.  If we desire to directly get a rank maximal assignment, completely ignoring the utilities, then we know this is polynomial by a result from \citet{GKKM+10a}.  However, if we wish to modulate between using the utilities and using only the ranks, perhaps we can use OWAs. We use the sum over all agents of $OWA_{\owa}(\wrs) = \owa \cdot \wrs$ as the optimization criteria for the assignment.

In order to cleanly define this we need to place some restrictions on our OWA vectors.  Firstly, the length of $\owa$ needs to be at least as long as the maximum agent capacity, i.e., $|\owa| \geq \arg\max_{i \in N}(\acmax(i))$.  Typically the literature on OWAs assumes that $\owa$ is normalized, i.e., $\sum_{1 \leq i \leq K} \alpha_i = 1$. We do not enforce this convention as we wish to study the PAV setting with $\alpha = (1, \nicefrac{1}{2}, \ldots)$.  This is formally a relaxation and we observe that whether or not the OWAs are normalized does not affect our computational results.  However, we do require that our OWA vector be non-increasing and that each entry be $\geq 0$, i.e., for any $i,j \in |\owa|$, $i < j$ we have $\owa_i \geq \owa_j \geq 0$.

\pbDef{\textsc{\sumowa Assignment}}
{Given an assignment setting $(N, O, \succ, \uv, \maxv)$ with agent capacities $[\acmin(i), \acmax(i)]$ for all $i \in N$, and object capacities $[\ocmin(j), \ocmax(j)]$ for all $j \in O$, and a non-increasing OWA vector $\owa_i$ with $|\owa| \geq \max_{\forall_i \in N}(\acmax(i))$.}
{Find a feasible assignment $A$ such that  
$$A = \arg\max_{A \in \all} \sum^{|n|}_{i=1} \owa_i \cdot \wrs_i(A).$$ }

In our formulation, the OWA operator is applied to the vector of agent utilities and then we aggregate (or sum) these modified utilities to give the assignment objective.  Hence, the \sumowa name.  We observe that this formulation strictly generalizes the utilitarian assignment objective; if we set $\owa = (1)^{n}$ we recover the utilitarian assignment.  

One may also wish to consider applying the OWA over the sorted vector of total agent utility for their allocation, which one could call the OWA-$\Sigma$ version of our problem.  Indeed, this formulation of the problem has been considered before and proposed in the earliest writings on OWAs for decision making \cite{Yager88a}.  Taking the OWA-$\Sigma$ formulation allows one to recover both the utilitarian assignment, $\owa = (\nicefrac{1}{n},\ldots, \nicefrac{1}{n})$, as well as the egalitarian assignment, $\owa = (0,\ldots,0_{n-1},1)$.  However, because the OWA-$\Sigma$ formulation is a generalization of the egalitarian assignment, it becomes NP-hard in general \cite{DeHi88a}.

We think of the $\alpha$ vector as a kind of control knob given to the implementer of the market, allowing them to apply a sub-linear transform to the agent utilities.  This ability may be especially useful when agents are free to report their (normalized) utilities for ranks via bidding or other mechanisms \cite{BuCa12a}.  In many settings the  utility vector is controlled by the individual \emph{agents}, while the OWA vector is under the control of the \emph{market implementers}.  Consider the following example.

\begin{example}\label{ex:disjoint}
Consider a setting with four agents $N=\{a_1,a_2,a_3,a_4\}$ agents and four objects $O = \{o_1,o_2,o_3,o_4\}$.  For all agents let $\acmin = \acmax = 2$ and for all objects let $\ocmin = \ocmax = 2$.  For the \sumowa assignment, let $\alpha = (1, \nicefrac{1}{2})$.

\begin{center}
\begin{tabular}{ccccc}
\hline
		& $o_1$	& $o_2$	& $o_3$	& $o_4$ \\ \hline	
$a_1$	& 11	& 9		& 0		& 0		\\
$a_2$	& 8		& 8		& 2		& 2		\\ 
$a_3$	& 7		& 7		& 3		& 3		\\ 
$a_4$	& 6		& 6		& 4		& 4		\\ \hline
\end{tabular}
\end{center}
\noindent
We get the following allocations.\\
\noindent
Utilitarian: \\
\indent
$A(a_1, :) = \{o_1, o_2\}, u_1(A) = 20$; 
$A(a_2, :) = \{o_1, o_2\}, u_2(A) = 16$; \\
\indent
$A(a_3, :) = \{o_3, o_4\}, u_3(A) = 6$; 
$A(a_4, :) = \{o_3, o_4\}, u_4(A) = 8$; \\
\indent
$\sum_i u_i(A) = 50$.\\

OWA, $\alpha = (1, \nicefrac{1}{2})$:\\
\indent
$A(a_1, :) = \{o_1, o_2\}, u_1(A) = 20$, $\alpha \cdot \wrs_1 = 15.5$; \\
\indent
$A(a_2, :) = \{o_2, o_3\}, u_2(A) = 10$, $\alpha \cdot \wrs_2 = 9.0$; \\
\indent
$A(a_3, :) = \{o_1, o_4\}, u_3(A) = 10$, $\alpha \cdot \wrs_3 = 8.5$; \\
\indent
$A(a_4, :) = \{o_3, o_4\}, u_4(A) = 8$, $\alpha \cdot \wrs_4 = 5.0$; \\
\indent
$\sum_i u_i(A) = 48$.\\

Egalitarian: \\
\indent
$A(a_1, :) = \{o_1, o_4\}, u_1(A) = 11$; 
$A(a_2, :) = \{o_2, o_4\}, u_2(A) = 10$; \\
\indent
$A(a_3, :) = \{o_2, o_3\}, u_3(A) = 10$;
$A(a_4, :) = \{o_1, o_3\}, u_4(A) = 10$; \\
\indent
$\sum_i u_i(A) = 41$.\\

\end{example}

Inspecting the results of Example \ref{ex:disjoint}, we observe that in the set of all utilitarian maximal assignments have $a_1$ and $a_2$ each being assigned to $o_1$ and $o_2$, in the set of all \sumowa maximal assignments $a_3$ is assigned one of $o_1$ or $o_2$ while $a_2$ is assigned one of $o_3$ or $o_4$, while in the set of all egalitarian maximal assignments each of the agents receives one of either $o_1$ or $o_2$ along with one of $o_3$ or $o_4$.  Thus we observe the following.

\begin{observation}
The set of assignments returned by each of the three objective functions, utilitarian, egalitarian, and OWA, can be disjoint.
\end{observation}

There are instances where the set of \sumowa assignments is the same as the set of egalitarian assignments, but disjoint from the set of utilitarian assignments.  Hence, it is an interesting direction for future work to fully characterize \sumowa assignments and discover OWA vectors with nice properties.

\subsection{Pareto Optimality}
An allocation $S$ is \emph{more preferred by a given agent with respect to pairwise comparisons} than allocation $T$ if $S$ is a result of replacing an item in $T$ with a strictly more preferred item.
Note that the pairwise comparison relation is transitive. An allocation is \emph{Pareto optimal with respect to pairwise comparisons} if there exists no other allocation that each agent weakly prefers and at least one agent strictly prefers.

\begin{lemma}\label{lem:p}
Consider an agent $i$ and two allocations $S$ and $T$ of equal size. Then if $S$ is at least as preferred as $T$ by $i$ with respect to pairwise comparison, then $S$ yields at least as much OWA value as $T$ for \emph{any} OWA vector no matter if it is increasing or decreasing.
\end{lemma}
\begin{proof}
Note that $S$ can be viewed as a transformation from $T$ where each item $j$ is replaced by some other item $j'$ that is at least as preferred. Hence, the value of the item either stays the same or increases. In either case, the corresponding OWA multiplied with the value is the same. Since the OWA transform is bilinear, the total OWA score of $S$ is at least as much as that of $T$.
\end{proof}
        
\begin{proposition}
The \sumowa maximal assignment is Pareto optimal with respect to pairwise comparison irrespective of the OWA. 
\end{proposition}
\begin{proof}
Assume for contradiction that a \sumowa maximal assignment $A$ is not Pareto optimal with respect to pairwise comparisons. From Lemma \ref{lem:p}, there exists another outcome $A'$ that each agent weakly prefers and at least one agent strictly prefers. But this means that in $A'$ each agent gets at least as much OWA score and at least one agent gets strictly more. But this contradicts the fact that $A$ is OWA maximal. 
\end{proof}

%
%
%

\section{An Algorithm for \sumowa assignments}\label{sec:algorithm}

We give an algorithm for finding \sumowa assignments using flow networks.  In this proof we use the most general formulation of our problem by allowing the values of the upper and lower per-agent capacities, $[\acmin(i), \acmax(i)]$, to vary for each agent; and the upper and lower object capacities, $[\ocmin(j), \ocmax(j)]$, to vary for each object.   

\begin{theorem}\label{thm:owa-p}
An \sumowa assignment can be found in polynomial time.
\end{theorem}

\begin{proof}
We reduce our problem to the problem of finding a minimum cost feasible flow in a graph with upper and lower capacities on the edges, which is a polynomial time solvable problem.  In addition to being polynomial time solvable, we know that the flow is integral as long as all edge capacities are integral, even if we have real valued costs \cite{AMO93a}.
Figures \ref{fig:main} and \ref{fig:agent} provide a high level view of the flow network that we will construct.

\begin{figure}
\centering
\scalebox{0.68}{
\begin{tikzpicture}[
    >=stealth',
    shorten >= 1pt,
    auto,
    node distance=2cm,
    semithick,
    bend angle=10,
    graybox/.style = {draw=gray!20, fill=gray!20, rounded corners},
    line/.style = {->, draw=black, thick},
    box/.style = {circle, draw=blue!50, fill=blue!20, minimum size=4mm},
    ball/.style = {ellipse, draw=blue!50, very thick, dashed, fill=blue!10, minimum size=4mm}
    ]
 
    \coordinate (S) at (-6cm, 0cm);
 
    \coordinate (S1) at (-3cm, 2cm);
    \coordinate (S2) at (-3cm, 1cm);
    \coordinate (S3) at (-3cm, 0cm);
    \coordinate (S4) at (-3cm, -2cm);
    
    \coordinate (G1) at (0cm, 2cm);
    \coordinate (G2) at (0cm, 1cm);
    \coordinate (G3) at (0cm, 0cm);
    \coordinate (G4) at (0cm, -2cm);
 
    \coordinate (M1) at (3cm, 2cm);
    \coordinate (M2) at (3cm,  1cm);
    \coordinate (M3) at (3cm, 0cm);
    \coordinate (M4) at (3cm, -2cm);
 
    \coordinate (T) at (6cm, 0cm);
 
    \node (BBox) [graybox, minimum width=1cm, minimum height=5cm] at (-3cm, 0cm) {};
    \node [above] at (BBox.90) {$N$};
 
    \node (BBox2) [graybox, minimum width=1cm, minimum height=5cm] at (3cm, 0cm) {};
    \node [above] at (BBox2.90) {$O$};
 
    \node (Sbox)  [box] at (S)  {$s$};
    \node (Sbox1) [box] at (S1) {$a_1$};
    \node (Sbox2) [box] at (S2) {$a_2$};
    \node (Sbox3) [box] at (S3) {$a_3$};
    \node (Sbox4) [box] at (S4) {$a_n$};
    
    \node (Ga1) [ball] at (G1) {Gadget $a_1$};
    \node (Ga2) [ball] at (G2) {Gadget $a_2$};
    \node (Ga3) [ball] at (G3) {Gadget $a_3$};
    \node (Ga4) [ball] at (G4) {Gadget $a_n$};
     
    \node (Mbox1) [box] at (M1) {$o_1$};
    \node (Mbox2) [box] at (M2) {$o_2$};
    \node (Mbox3) [box] at (M3) {$o_3$};
    \node (Mbox4) [box] at (M4) {$o_m$};
     
    \node (Tbox) [box] at (T) {$t$};
     
    \fill [black] (-3cm,  -.8cm) circle (1.2pt);
    \fill [black] (-3cm,   -1cm) circle (1.2pt);
    \fill [black] (-3cm, -1.2cm) circle (1.2pt);
 
    \fill [black] (3cm, -.8cm) circle (1.2pt);
    \fill [black] (3cm, -1cm) circle (1.2pt);
    \fill [black] (3cm, -1.2cm) circle (1.2pt);

    \path[line] (Sbox) |- node [pos=0.7, above, sloped] {\scriptsize $C = [\acmin(1), \acmax(1)]$} (Sbox1);
    \path[line] (Sbox) -- node [pos=0.5, above, sloped] {\scriptsize $C = [\acmin(2), \acmax(2)]$} (Sbox2);
    \path[line] (Sbox) -- node [pos=0.5, below, sloped] {\scriptsize $C = [\acmin(3), \acmax(3)]$} (Sbox3);
    \path[line] (Sbox) |- node [pos=0.7, below, sloped] {\scriptsize $C = [\acmin(n), \acmax(n)]$} (Sbox4);
 
       \path[line] (Sbox1.55) -- node [] {} (Ga1.160);
       \path[line] (Sbox1.34) -- node [] {} (Ga1.170);
       \path[line] (Sbox1.-55) -- node [pos=0.2, above] {\textbf{\vdots}} (Ga1.200);
       
      \path[line] (Sbox2.55) -- node [] {} (Ga2.160);
       \path[line] (Sbox2.34) -- node [] {} (Ga2.170);
       \path[line] (Sbox2.-55) -- node [pos=0.2, above] {\textbf{\vdots}} (Ga2.200);
       
      \path[line] (Sbox3.55) -- node [] {} (Ga3.160);
       \path[line] (Sbox3.34) -- node [] {} (Ga3.170);
       \path[line] (Sbox3.-55) -- node [pos=0.2, above] {\textbf{\vdots}} (Ga3.200);
       
      \path[line] (Sbox4.55) -- node [] {} (Ga4.160);
       \path[line] (Sbox4.34) -- node [] {} (Ga4.170);
       \path[line] (Sbox4.-55) -- node [pos=0.2, above] {\textbf{\vdots}} (Ga4.200);
       
       \path[line] (Ga1.10) -- node [] {} (Mbox1);
       \path[line] (Ga1.5) -- node [] {} (Mbox2);
       \path[line] (Ga1.-5) -- node [] {} (Mbox3);
       \path[line] (Ga1.-10) -- node [] {} (Mbox4);
       
       \path[line] (Ga2.10) -- node [] {} (Mbox1);
       \path[line] (Ga2.5) -- node [] {} (Mbox2);
       \path[line] (Ga2.-5) -- node [] {} (Mbox3);
       \path[line] (Ga2.-10) -- node [] {} (Mbox4);
       
       \path[line] (Ga3.10) -- node [] {} (Mbox1);
       \path[line] (Ga3.5) -- node [] {} (Mbox2);
       \path[line] (Ga3.-5) -- node [] {} (Mbox3);
       \path[line] (Ga3.-10) -- node [] {} (Mbox4);
       
       \path[line] (Ga4.10) -- node [] {} (Mbox1);
       \path[line] (Ga4.5) -- node [] {} (Mbox2);
       \path[line] (Ga4.-5) -- node [] {} (Mbox3);
       \path[line] (Ga4.-10) -- node [] {} (Mbox4);       
 
    \path[line] (Mbox1) -| node [pos=0.3, above, sloped] {\scriptsize $C = [\ocmin(1), \ocmax(1)]$} (Tbox);
    \path[line] (Mbox2) -- node [pos=0.5, above, sloped] {\scriptsize $C = [\ocmin(2), \ocmax(2)]$} (Tbox);
    \path[line] (Mbox3) -- node [pos=0.5, below, sloped] {\scriptsize $C = [\ocmin(3), \ocmax(3)]$} (Tbox);
    \path[line] (Mbox4) -| node [pos=0.3, below, sloped] {\scriptsize $C = [\ocmin(m), \ocmax(m)]$} (Tbox);

\end{tikzpicture}
}
\caption{Main gadget for the reduction which enforces the agent and object capacity constraints.}
\label{fig:main}
\end{figure}
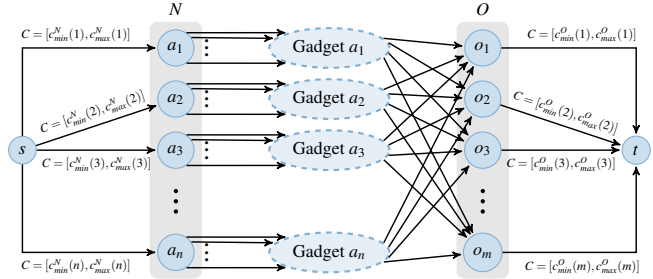

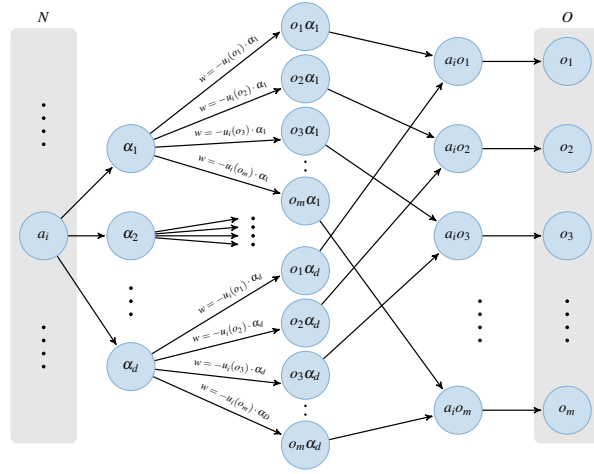
\begin{figure}
\centering
\scalebox{0.58}{
\begin{tikzpicture}[
    >=stealth',
    shorten >= 1pt,
    auto,
    node distance=2cm,
    semithick,
    bend angle=10,
    graybox/.style = {draw=gray!20, fill=gray!20, rounded corners},
    line/.style = {->, draw=black, thick},
    box/.style = {circle, draw=blue!50, fill=blue!20, minimum size=11mm, inner sep=0pt},
    ball/.style = {ellipse, draw=blue!50, very thick, dashed, fill=blue!10, minimum size=4mm}
    ]
 
    \coordinate (A) at (-6cm, 0cm);
 
    \coordinate (C1) at (-4cm, 2cm);
    \coordinate (C2) at (-4cm, 0cm);
    \coordinate (C4) at (-4cm, -3cm);
        
    \coordinate (G1) at (0cm, 4.8cm);
    \coordinate (G2) at (0cm, 3.6cm);
    \coordinate (G3) at (0cm, 2.4cm);
    \coordinate (G4) at (0cm, 0.8cm);
   
    \coordinate (d1) at (-1.5cm, 0.4cm);
    \coordinate (d2) at (-1.5cm, 0.2cm);
    \coordinate (d3) at (-1.5cm, 0cm);
    \coordinate (d4) at (-1.5cm, -0.2cm);

    \coordinate (G5) at (0cm, -0.8cm);
    \coordinate (G6) at (0cm, -2.0cm);
    \coordinate (G7) at (0cm, -3.2cm);
    \coordinate (G8) at (0cm, -4.8cm);
 
    \coordinate (AO1) at (3.5cm, 4cm);
    \coordinate (AO2) at (3.5cm,  2cm);
    \coordinate (AO3) at (3.5cm, 0cm);
    \coordinate (AO4) at (3.5cm, -4cm);
 
    \coordinate (O1) at (6cm, 4cm);
    \coordinate (O2) at (6cm,  2cm);
    \coordinate (O3) at (6cm, 0cm);
    \coordinate (O4) at (6cm, -4cm);

    \node (BBox) [graybox, minimum width=1.5cm, minimum height=9.5cm] at (-6cm, 0cm) {};
    \node [above] at (BBox.90) {$N$};
 
    \node (BBox2) [graybox, minimum width=1.5cm, minimum height=9.5cm] at (6cm, 0cm) {};
    \node [above] at (BBox2.90) {$O$};
 
    \node (Sbox)  [box] at (A)  {$a_i$};
    
    \node (Sbox1) [box] at (C1) {$\alpha_1$};
    \node (Sbox2) [box] at (C2) {$\alpha_2$};
    \node (Sbox4) [box] at (C4) {$\alpha_{d}$};

    \node (oc1) [box] at (G1) {$o_1\alpha_1$};
    \node (oc2) [box] at (G2) {$o_2\alpha_1$};
    \node (oc3) [box] at (G3) {$o_3\alpha_1$};
    \node (oc4) [box] at (G4) {$o_m\alpha_1$};
    
    \node (oc5) [box] at (G5) {$o_1\alpha_{d}$};
    \node (oc6) [box] at (G6) {$o_2\alpha_{d}$};
    \node (oc7) [box] at (G7) {$o_3\alpha_{d}$};
    \node (oc8) [box] at (G8) {$o_m\alpha_{d}$};
     
    \node (ao1) [box] at (AO1) {$a_io_1$};
    \node (ao2) [box] at (AO2) {$a_io_2$};
    \node (ao3) [box] at (AO3) {$a_io_3$};
    \node (ao4) [box] at (AO4) {$a_io_m$}; 
     
    \node (Mbox1) [box] at (O1) {$o_1$};
    \node (Mbox2) [box] at (O2) {$o_2$};
    \node (Mbox3) [box] at (O3) {$o_3$};
    \node (Mbox4) [box] at (O4) {$o_m$};
     

    \fill [black] (0cm,  1.7cm) circle (0.8pt);
    \fill [black] (0cm,   1.5cm) circle (0.8pt);
    
    \fill [black] (0cm,  -3.9cm) circle (0.8pt);
    \fill [black] (0cm,   -4.1cm) circle (0.8pt);
     
    \fill [black] (-4cm,  -1.2cm) circle (1.2pt);
    \fill [black] (-4cm,   -1.5cm) circle (1.2pt);
    \fill [black] (-4cm, -1.8cm) circle (1.2pt);
 
    \fill [black] (-1.2cm,  .4cm) circle (1.2pt);
    \fill [black] (-1.2cm,   .2cm) circle (1.2pt);
    \fill [black] (-1.2cm, 0cm) circle (1.2pt);
    \fill [black] (-1.2cm, -.2cm) circle (1.2pt);
 
    \fill [black] (4cm, -1.5cm) circle (1.2pt);
    \fill [black] (4cm, -1.8cm) circle (1.2pt);
    \fill [black] (4cm, -2.1cm) circle (1.2pt);
    \fill [black] (4cm, -2.4cm) circle (1.2pt);
 
    \fill [black] (6cm, -1.5cm) circle (1.2pt);
    \fill [black] (6cm, -1.8cm) circle (1.2pt);
    \fill [black] (6cm, -2.1cm) circle (1.2pt);
    \fill [black] (6cm, -2.4cm) circle (1.2pt);
    
    \fill [black] (-6cm, 3cm) circle (1.2pt);
    \fill [black] (-6cm, 2.7cm) circle (1.2pt);
    \fill [black] (-6cm, 2.4cm) circle (1.2pt);
    \fill [black] (-6cm, 2.1cm) circle (1.2pt);
    
    \fill [black] (-6cm, -3cm) circle (1.2pt);
    \fill [black] (-6cm, -2.7cm) circle (1.2pt);
    \fill [black] (-6cm, -2.4cm) circle (1.2pt);
    \fill [black] (-6cm, -2.1cm) circle (1.2pt);

    \path[line] (Sbox) -- node [] {} (Sbox1);
    \path[line] (Sbox) -- node [] {} (Sbox2);
    \path[line] (Sbox) -- node [] {} (Sbox4);
    
    \path[line] (Sbox1) -- node [pos=0.65, above, sloped] {\scriptsize $w = -u_i(o_1) \cdot \alpha_1$} (oc1.190);
    \path[line] (Sbox1) -- node [pos=0.65, above, sloped] {\scriptsize $w = -u_i(o_2) \cdot \alpha_1$} (oc2);
    \path[line] (Sbox1) -- node [pos=0.6, above, sloped] {\scriptsize $w = -u_i(o_3) \cdot \alpha_1$} (oc3.200);
    \path[line] (Sbox1) -- node [pos=0.6, above, sloped] {\scriptsize $w = -u_i(o_m) \cdot \alpha_1$} (oc4);
    
    \path[line] (Sbox2) -- node [] {} (d1);
    \path[line] (Sbox2) -- node [] {} (d2);
    \path[line] (Sbox2) -- node [] {} (d3);
    \path[line] (Sbox2) -- node [] {} (d4);

    \path[line] (Sbox4) -- node [pos=0.65, above, sloped] {\scriptsize $w = -u_i(o_1) \cdot \alpha_{d}$} (oc5.190);
    \path[line] (Sbox4) -- node [pos=0.60, above, sloped] {\scriptsize $w = -u_i(o_2) \cdot \alpha_{d}$} (oc6);
    \path[line] (Sbox4) -- node [pos=0.58, above, sloped] {\scriptsize $w = -u_i(o_3) \cdot \alpha_{d}$} (oc7.200);
    \path[line] (Sbox4) -- node [pos=0.6, above, sloped] {\scriptsize $w = -u_i(o_m) \cdot \alpha_{D}$} (oc8);

    \path[line] (oc1) -- node [pos=0.4, above] {} (ao1);
    \path[line] (oc2) -- node [pos=0.4, above] {} (ao2);
    \path[line] (oc3) -- node [pos=0.4, above] {} (ao3);
    \path[line] (oc4) -- node [pos=0.4, above] {} (ao4);
 
    \path[line] (oc5) -- node [pos=0.4, above] {} (ao1);
    \path[line] (oc6) -- node [pos=0.4, above] {} (ao2);
    \path[line] (oc7) -- node [pos=0.4, above] {} (ao3);
    \path[line] (oc8) -- node [pos=0.4, above] {} (ao4);

   
    \path[line] (ao1) -- node [] {} (Mbox1);
    \path[line] (ao2) -- node [] {} (Mbox2);
    \path[line] (ao3) -- node [] {} (Mbox3);
    \path[line] (ao4) -- node [] {} (Mbox4);

\end{tikzpicture}
}
\caption{The per agent gadget. Note that all costs on edges are $0$ and all capacities are [0,1] unless otherwise noted.}
\label{fig:agent}
\end{figure}

In Figure \ref{fig:main} we first build a tripartite graph with two sets of nodes and one set of gadgets per agent: the agent nodes, one for each agent $a_i$; the agent gadgets, one (illustrated in Figure \ref{fig:agent}) for each agent $a_i$; and the object nodes, one for each object $o_j$. There is an edge from the source node $s$ to each of the agent nodes, each with cost 0, minimum flow capacity $\acmin(i)$ and a maximum flow capacity $\acmax(i)$.  This set of edges and nodes enforces the constraint that each $a_i$ has capacity $[\acmin(i), \acmax(i)]$.  We also construct an edge from each object node to the sink $t$.  Each of these edges has a cost 0, a minimum capacity $\ocmin(j)$, and a maximum capacity $\ocmax(j)$.  This set of edges enforces the constraint that each $o_j$ has capacity $[\ocmin(j), \ocmax(j)]$.

We now turn to the agent gadget depicted in Figure \ref{fig:agent} for arbitrary $a_i$. The leftmost node and the rightmost set of nodes in Figure \ref{fig:agent} correspond to the agent nodes $N$ and object nodes $O$ in Figure \ref{fig:main}, respectively. In each agent gadget we create a tripartite sub-graph with the agent node $a_i$ serving as the source and the set of object nodes $O$ serving as the sinks.

We create three layers of nodes which we describe in turn from left to right.  First, we create a set of decision nodes with labels $\alpha_1, \ldots, \alpha_{d}$ where $ \acmax(i) \leq d \leq |\alpha|$. Intuitively, we will be multiplying the OWA value $\alpha_1$ by the utility for some object, so we need to keep track of all the values that could result.  The arcs from $a_i$ to each of the nodes in this set has upper capacity 1, minimum capacity 0, and cost 0.  If we have the case that $\acmax(i) < d$ then we can set the maximum capacity of the edges to node(s) $\alpha_j, j>\acmax(i)$ to 0.  This enforces that each value in the OWA vector can modify at most one utility value.
  
For each of the decision nodes $\alpha_1, \ldots, \alpha_{d}$ constructed, we create a set of object/decision nodes for each $o_j$ which we denote $o_j\alpha_k$.  From each of the decision nodes $\alpha_1, \ldots, \alpha_{d}$ we create an edge to each of the object/decision nodes created for this particular decision node $\alpha_k$, i.e., $o_1\alpha_1, o_2\alpha_1, \ldots, o_m\alpha_1$ for $\alpha_1$.  Each of these edges has maximum capacity 1 and a cost equal to $-1 \cdot u_i(o_j) \cdot \alpha_1$ for rank 1 and object $o_j \in O$. These costs are the (negative) cost that matching agent $a_i$ with object $o_j$ at weighted rank $d_k$ contributes to the OWA objective.

Finally, we create one set of agent/object nodes, one for each $o_j$ denoted $a_io_j$. From all the object/decision nodes we connect all nodes with a label of $o_j$ to the corresponding agent/paper node, i.e., $o_1\alpha_1, o_1\alpha_2, \ldots o_1\alpha_d$ all connect to $a_io_1$ with cost 0 and maximum capacity 1. We then connect the agent/object node to the corresponding object node in the main construction from $O$, i.e., $a_io_1$ to $o_1$ with cost 0 and maximum capacity 1.  This set of nodes and edges enforces that each agent can be assigned each object once.

We can extract an assignment from the minimum cost feasible flow by observing that paper $o_j$ is allocated to agent $a_i$ if and only if there is a unit of flow passing from the particular agent/object node $a_io_j$ to the object node $o_j$.  We now argue for the correctness of our algorithm in two steps, (1) that all constraints for the \sumowa assignment problem are enforced and (2) that a minimum cost feasible flow in the constructed graph gives an \sumowa assignment. For (1) we note that since the units of flow across the graph represent the assignment and we have explained how the capacity constraints on all edges enforce each of the particular constraints imposed by our definition of a feasible assignment, there is a feasible flow iff the flow satisfies the constraints.

For (2) observe that for each agent, the $\alpha$ nodes fill with flow in order from $\alpha_1$ to $\alpha_d$ as the OWA vector is non-increasing and the utilities are decreasing, i.e., for each agent, the edge costs monotonically increase from the edges associated with $\alpha_1$ to the edges associated with $\alpha_d$.  Thus, for each agent, the first unit of flow to this agent will use the least cost (most negative) edge must be associated with $\alpha_1$; and similarly for $\alpha_{2}$ through $\alpha_{d}$.  From the capacity constraints we know there is only one unit of flow that enters each decision node $\alpha_i$ and there is only one unit of flow that can leave each agent/paper node $a_io_j$.  This means that each $\alpha_i$ can modify only one $o_j$ and each $o_j$ selected must be unique for this agent.

As the decision nodes are filled in order and $\alpha_i$ can only modify the value for a single object, we know the total cost of the flow across the agent gadget for each $a_i$ is equal to $-1 \cdot \owa_i \cdot \wrs_i$.  Hence, the price of the min cost flow across all agents is equal to $-1 \cdot \sum_{\forall i \in N} \owa_i \cdot \wrs_i(A)$. Thus, the min cost flow in the graph is an \sumowa assignment.
\end{proof}

\subsection{Generalizations}

We observe two possible generalizations of the above construction which allow us to use this constructive proof for more general instances than the CPAP. First, The proof above can be generalized to allow for $\alpha$ to vary for each agent.  Specifically, observe that the decision nodes for each agent $a_i$ are independent from all other agents.  This means that, for each agent (or a class of agents) we could use an OWA vector $\alpha^{a_{i}}$.  This ability may be useful, for instance, when a group of agents reports the same extreme utility distribution and the organizer wishes to apply the same transform to these utilities.  

The second generalization that we can make to the above construction is to allow each agent to be assigned to each object more than once.  While this ability does not make sense in the reviewers/papers setting (unless there are sub reviewers) there could be other capacitated assignment settings where we may wish to assign the agents to objects multiple times e.g., if there are discrete jobs that need to be done a certain number of times but and a single agent can be assigned the same job multiple times.

In order to generalize the capacity constraint from 1 for each agent $i$ for each object $j$ we introduce a capacity upper bound $z_{i,j}$ which encodes the number of times that agent $i$ can be assigned to object $j$.  Taking $z_{i,j} = 1$ for all $i$ and $j$ gives us the original CPAP setting.  In order to enforce this constraint, within each agent gadget (Figure \ref{fig:agent}) we add a capacity constraint equal to $z_{i,j}$ from each edge $a_io_j$ to $o_j$.  If we want a lower bound for the number of copies of $o_j$ assigned to $a_i$ we can encode this lower bound on this edge as well.

We can extract an assignment from the minimum cost feasible flow by observing that paper $o_j$ is allocated to agent $a_i$ $z_{ij}$ times if and only if there are units of flow passing from the particular agent/object node $a_io_j$ to the object node $o_j$.  The argument for correctness follows exactly from the proof of Theorem \ref{thm:owa-p} above.

\begin{corollary}
An \sumowa assignment can be found in polynomial time even if each agent $a_i$ has a unique OWA vector $\alpha^{a_i}$ and each object $o_j$ can be assigned to each agent $a_i$ any number of times (not just once).
\end{corollary}

\section{Experiments}
\label{sec:experiments}

We now turn to the question of how good are \sumowa assignments in practice? We answer this question using real world data from three large international conferences (MD-00002-00000001 -- 00000003) from \textsc{www.PrefLib.org} \cite{MaWa13a}. We focus discussion on MD-00002-00000003 which has 146 agents and 175 objects. We implemented the algorithm given in Section \ref{sec:algorithm} using networkX for Python and Lemon for C++.  However, we still have a run time $\approx O(V^4)$, giving runtime $\approx (150^2 \cdot 3)^4) = 2 \times 10^{19}$, which caused our computers to crash even with 16GB of memory.  This was quite disappointing as we thought the flow argument could be used to solve this problem on real-world instances.

Not to be deterred, we still wanted to investigate the assignments we get from \sumowa   compare to the utilitarian and egalitarian assignments.  Consequently, we implemented the model as an MIP in Gurobi 7.0 and it ran in under 1 minute for all instances and settings using 4 cores. Our MIP is similar to the one given by \citet{SFL16a} and the MARA MIP by \citet{BoChLa16}.  However, as we have capacity constraints and individual/variable length OWAs, our MIP is more general than either.

To encode the \sumowa problem we introduce a binary variable $x_{a,o}$ indicating that agent $a$ is assigned object $o$.  We introduce a real valued variable $u_{owa,a}$ which is the \sumowa utility for agent $a$.
Finally, we introduce $r_{a,o,p}$ for the OWA matrix which notes that agent $a$ is assigned object $o$ at OWA rank $p$.  The MIP is given below.

{\small
\vspace{-1.5mm}
\[
\begin{array}{|ll|l|l|}%
\hline
\text{max} 		& \sum_{a \in A} \sum_{o \in O, p \in P} u_{a}(o) \cdot \alpha_{p} \cdot r_{a,o,p} & & \text{Description:}  \\
\text{s.t.} 	& \ocmin(o) \leq \sum_{a \in A} x_{a,o} \leq \ocmax(o) & \forall o \in O & \text{(1) Object Capacities}\\
				& \acmin(a) \leq \sum_{o \in O} x_{a,o} \leq \acmax(a) & \forall a \in A & \text{(2) Agent Capacities} \\
  				& \sum_{p \in P} r_{a,o,p} \leq 1						& \forall a \in A, \forall o \in O & \text{(3) One Object per OWA Rank}\\
  				& \sum_{o \in O} r_{a,o,p} \leq 1		 				& \forall a \in A, \forall p \in P & \text{(4) Objects Have One Rank} \\
				& \sum_{p \in P} r_{a,o,p} \geq x_{a,o}	 				& \forall a \in A, \forall o \in O & \text{(5) Assignment to OWA Link Fcn.} \\
				& \sum_{o \in O} r_{a,o,p} \geq \sum_{o \in O} r_{a,o,p+1}& \forall a \in A, \forall p \in P & \text{(6) Ranks Fill in Increasing Order} \\
				& \sum_{o \in O} r_{a,o,p} \cdot u_{a}(o) \geq \sum_{o \in O} r_{a,o,p+1} \cdot u_{a}(o) & \forall a \in A, \forall p \in P & \text{(7) Agent Utility Must Be Decreasing} \\
\hline
\end{array}
\]
}

Constraints (1)--(4) enforce the cardinality constraints on the agents, objects, and OWA rank matrix.  Constraint (5) links the agent and object assignments to be positions in the OWA rank matrix.  Line (6) enforces that the rank matrix fills from the first position to the $\acmax$ position for each agent.  And finally (7) enforces that the \sumowa value of the assignment positions in the rank matrix must be decreasing.  We then maximize the sum over all agents of the OWA objective value.

We found the utilitarian, egalitarian, and \sumowa assignments for each of the real world datasets when each object must receive 3--4 reviews and each agent must review 6--7 objects.  In the data, each agent sorts the papers into 4 equivalence classes which we gave utility values $(5, 3, 1, 0)$.  We use the PAV inspired decreasing harmonic OWA vector $(1, \nicefrac{1}{2}, \nicefrac{1}{3}, \ldots)$ to compute the \sumowa assignment.

One of the reasons we wanted to use the \sumowa assignment is to allow the market designer to enforce a more equitable distribution of papers with respect to the ranks. Hence, our test statistic is the number of top ranked items that the average agent can expect to receive.  Figure \ref{fig:cdf} shows the agent counts and the cumulative distribution function (CDF) for the number of top ranked items the agents receive.  

Looking at the left side of the figure, we see that 71 agents receive 5 top ranked papers under the \sumowa assignment while under the utilitarian assignment only 46 do.  Under the utilitarian assignment 35 agents receive more than 5 top ranked papers.  Consequently, on average, agents can expect to get 4 top ranked papers in the \sumowa assignment, 3 in the egalitarian assignment, and 4.2 in the utilitarian assignemnt.  However under the utilitarian assignment, several agents receive an entire set of top ranked objects, while the egalitarian assignment modulates this so that most agents only receive 3--4 top ranked items.  In contrast, the \sumowa assignment is a balance between these with the most agents receiving 5 top ranked items.

\begin{figure}[H] \centering
\includegraphics[width=\linewidth]{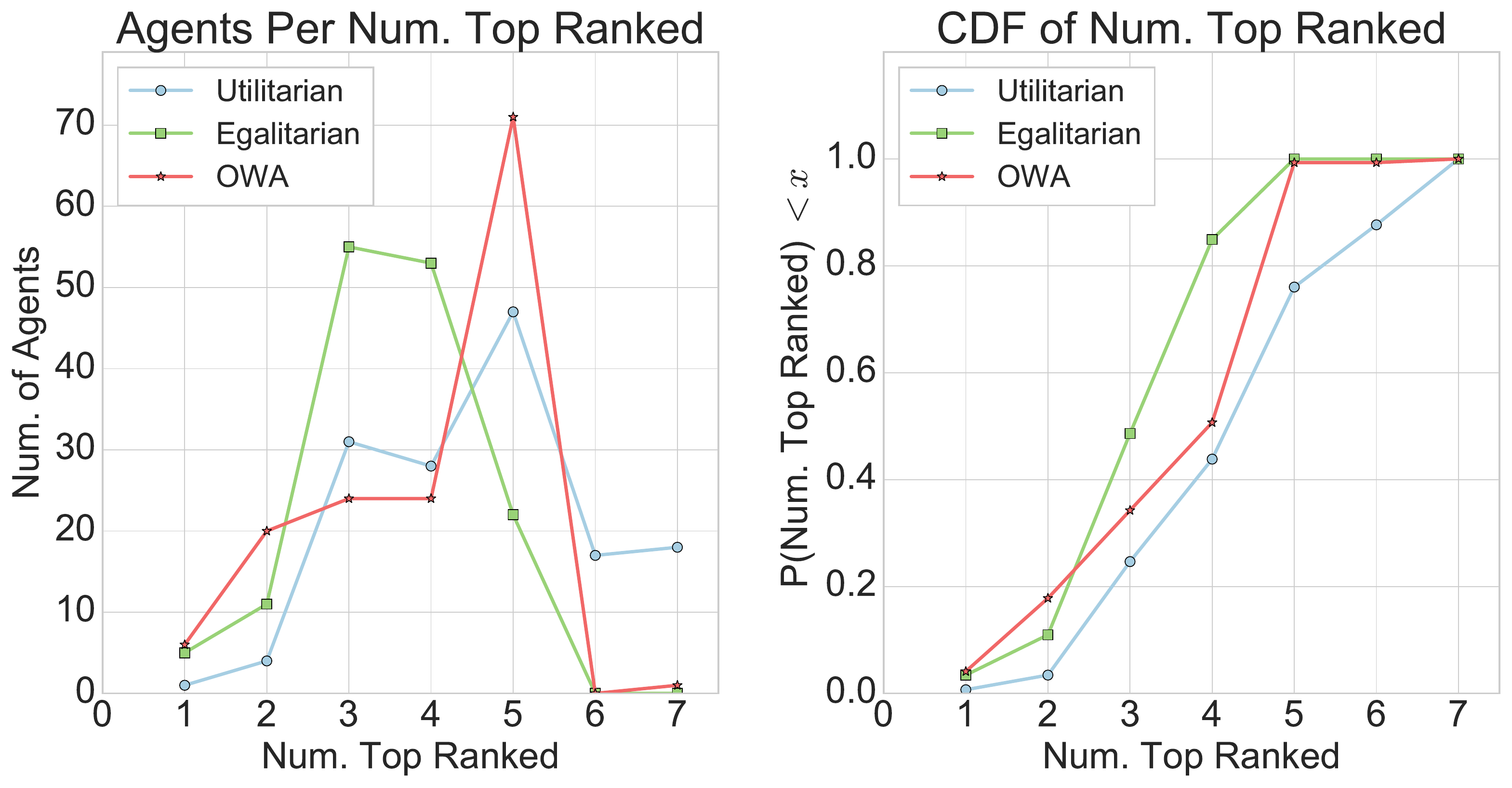}
\caption{The count of agents receiving $x$ top ranked papers (top) and the cumulative distribution function (CDF) (bottom) for the number of agents being assigned $x$ top ranked objects for MD-00002-00000003.  Though 100\% of the agents receive between 1 and 5 top ranked items for the egalitarian and \sumowa assignments (CDF), the PDF shows that the most agents receive the most top ranked items under the \sumowa assignment.}
\label{fig:cdf}
\end{figure}

\section{Conclusions}
We have proposed and provided algorithms for the novel notion of a \sumowa assignment. The \sumowa assignment using decreasing OWA vectors gives the central organizer a ``slider'' to move from utility maximizing towards a more rank maximal assignment computationally efficient package.  An important open question for future work is to find axiomatic characterizations for good OWA vectors.  Additionally, the OWA method, and all methods for CPAP that we surveyed, treat objects as having positive utility.  It is generally the case that reviewers at a conference want to review fewer, not more, papers.  Consequently it would be interesting to study CPAP from the point of view of \emph{chores}, as they are called in the economics literature.


\end{document}